%% file: main.tex
\documentclass[twoside]{article}

\usepackage[accepted]{aistats2023}
\usepackage{algorithm}
\usepackage{algorithmic}

\usepackage{tabularx}
\usepackage{amsfonts}
\usepackage{amsmath}
\usepackage{amssymb}
\usepackage[utf8]{inputenc}
\usepackage[english]{babel}
\usepackage{amsthm}
\usepackage{tabu}
\usepackage{url} 

\usepackage{caption, subcaption}
\usepackage{enumerate,enumitem,tikz,graphicx,svg, mathrsfs,eucal,verbatim, bbm, derivative,wrapfig}

\usepackage{makecell}
\usepackage[thinlines]{easytable}

\usepackage{rotating}
\usepackage{booktabs}
\usepackage{array,multirow,graphicx}

\theoremstyle{plain}% default
\newtheorem{thm}{Theorem}[section]
\newtheorem{assm}{Assumption}[section]

\newtheorem{prop}[thm]{Proposition}
\newtheorem*{nonumbprop}{Proposition}
\newtheorem*{cor}{Corollary}
\theoremstyle{definition}
\newtheorem{defn}{Definition}[section]
 \usepackage[
    backend=biber,
    style=authoryear,
    maxcitenames=1,
    uniquelist=false
  ]{biblatex}
\let\svthefootnote\thefootnote
\newcommand\freefootnote[1]{%
  \let\thefootnote\relax%
  \footnotetext{\hspace{-1.5em}#1}%
  \let\thefootnote\svthefootnote%
}

\begin{document}


% \runningauthor{Giovanni Luca Marchetti, Gustav T\'egner, }

\twocolumn[

\aistatstitle{Equivariant Representation Learning via Class-Pose Decomposition}

\aistatsauthor{\quad  Giovanni Luca Marchetti* \And Gustaf Tegn\'er* \And Anastasiia Varava \And Danica Kragic}

\aistatsaddress{School of Electrical Engineering and Computer Science, KTH Royal Institute of Technology \\
    Stockholm, Sweden } ]

\begin{abstract}
We introduce a general method for learning representations that are equivariant to symmetries of data. Our central idea is to decompose the latent space into an invariant factor and the symmetry group itself. The components semantically correspond to intrinsic data classes and poses respectively. The learner is trained on a loss encouraging equivariance based on supervision from relative symmetry information. The approach is motivated by theoretical results from group theory and guarantees representations that are lossless, interpretable and disentangled. We provide an empirical investigation via experiments involving datasets with a variety of symmetries. Results show that our representations capture the geometry of data and outperform other equivariant representation learning frameworks.   
\end{abstract}

\input{sections/introduction2}

\input{sections/method.tex}

\input{sections/related_work.tex}

\input{sections/experiments.tex}

% \bibsection
% \bibliography{main.bib}

\printbibliography

\clearpage
\appendix

\thispagestyle{empty}

% For one-column format, uncomment the following:
% For two-column format, uncomment the following:
%\twocolumn[ \makesupplementtitle ]

\input{sections/appendix.tex}

\end{document}

%% file: sections/introduction2.tex
\section{INTRODUCTION}
For an intelligent agent aiming to understand the world and to operate therein, it is crucial to construct rich \emph{representations} reflecting the intrinsic structures of the perceived data (\cite{bengio_disentanglement}). A variety of self-supervised approaches have been proposed to address the problem of representation learning such as (variational) auto-encoders (\cite{kingma2013auto, baldi2012autoencoders}) and contrastive learning methods (\cite{chen2020simple, jaiswal2020survey}). These approaches are applicable to a wide range of scenarios due to their generality but often fail to capture fundamental aspects hidden in data. For example, it has been recently shown that \emph{disentangling} intrinsic factors of variation requires specific biases or some form of supervision (\cite{challenging_disentanglement}). This raises the need for representation learning paradigms leveraging upon specific structures carried by data.
%Learning such a representation in an unsupervised manner is however not possible without imposing significant inductive bias on the model \cite{challenging_disentanglement}. In the case of object-recognition, one such bias could be spatial proximity 4

%or relies on enforcing pixel-based priors in the architecture [object-centric methods] that are prone to fail on more complex datasets.

\begin{figure}[t]
\centering
%\includesvg[width=.55\linewidth]{figures/firstpage.svg}
\includegraphics[width=.80\linewidth]{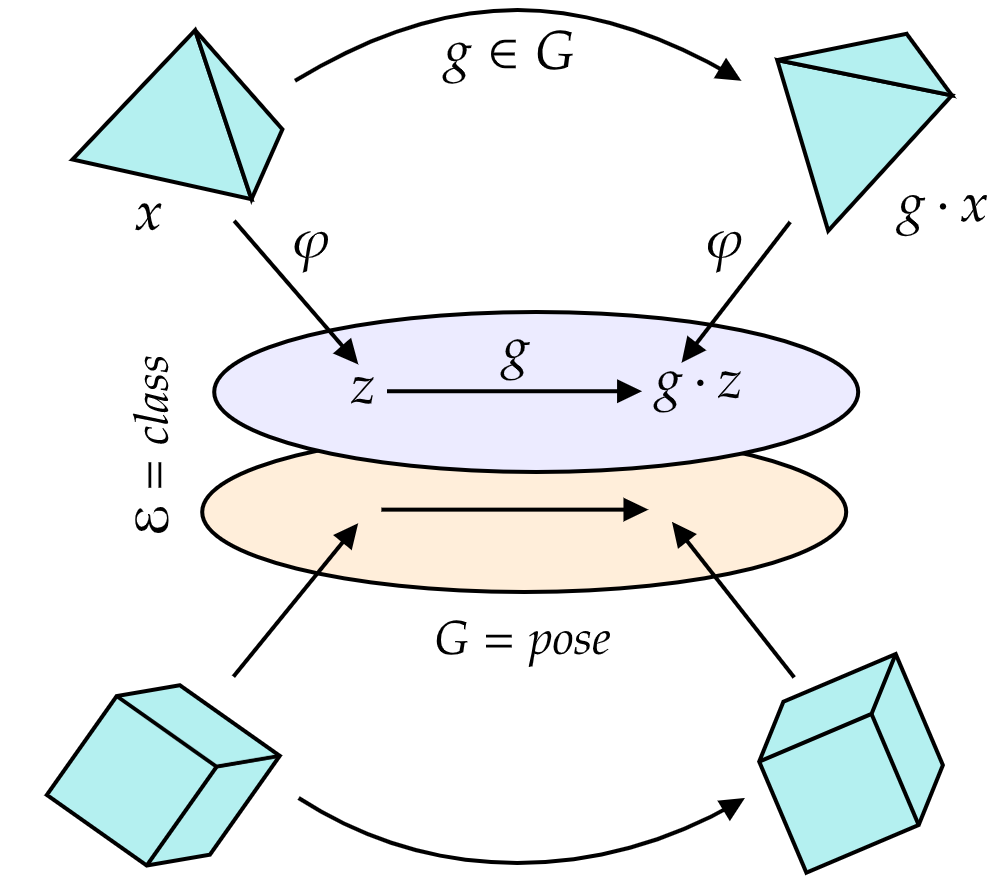}
\caption{An illustration of our equivariant representation decomposing intrinsic class and pose of data.}
\label{fig:equivariance}
\end{figure}

A fundamental geometric structure of datasets consists of their \emph{symmetries} (\cite{higgins2022symmetry}). Such structure arises in several practical scenarios. Consider for example images depicting rigid objects in different poses. In this case the symmetries are rigid transformations (translations and rotations) that \emph{act} on datapoints by transforming the pose of the depicted object. Symmetries not only capture the geometry of the pose but additionally preserve the object's shape, partitioning the dataset into intrinsic \emph{invariant} classes. The joint information of shape and pose describes the data completely and is recoverable from the symmetry structure alone. Another example of symmetries arises in the context of an agent exploring an environment. Actions performed by a mobile robot can be interpreted as changes of frame i.e., symmetries transforming the data the agent perceives (see Figure \ref{figagent}). Assuming the agent is capable of odometry (measurement of its own movement), such symmetries are collectable and available for learning.  All this motivates the design of representations that rely on symmetries and behave coherently with respect to them -- a property known as \emph{equivariance} (\cite{bengio2019meta, higgins2018towards}). 
\freefootnote{*Equal contribution.}

In this work we introduce a general framework for equivariant representation learning. Our central idea is to encode \emph{class and  pose separately} by decomposing the latent space into an invariant factor $\mathcal{E}$ and a symmetry component $G$ (see Figure \ref{fig:equivariance}). The pose component $G$ extracts geometry from data while the class component is interpretable and necessary for a lossless representation. We then train a representation learner $\varphi: \ \mathcal{X} \rightarrow \mathcal{E}\times G$ via a loss encouraging equivariance relying on supervision from relative symmetries between datapoints. Our methodology is based on a theoretical result guaranteeing that under mild assumptions an ideal learner achieves isomorphic representations by being trained on equivariance alone. Another advantage of our framework in the presence of multiple symmetry factors is that each of them can be varied independently by acting on the pose component. This realizes disentanglement in the sense of \cite{higgins2018towards} which, as mentioned, would not be possible without the information carried by symmetries. 

\begin{figure}
\centering
\includegraphics[width=.87\linewidth]{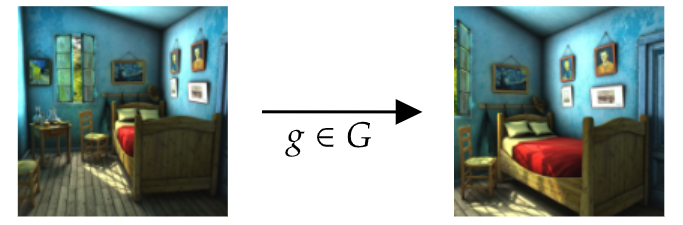}
\caption{Actions performed by a mobile agent can be seen as symmetries of the perceived data.}
\label{figagent}
\end{figure}

 We rely on the abstract language of (Lie) \emph{group theory} in order to formalize symmetries and equivariance. As a consequence, our framework is general and applicable to arbitrary groups of symmetries and to data of arbitrary nature. This is in contrast with previous works on equivariance often focusing on specific scenarios. For example, a number of works focus on Euclidean representations and linear or affine symmetry groups (\cite{guo2019affine, worrall2017interpretable, dynamic_enviroments}) while others enforce equivariance via group \emph{convolutions} (\cite{taco_group_equiv, cohengeneraltheory}). The latter are only applicable when data consists of signals over a base space $P$ (e.g., a pixel plane or a voxel grid) and symmetries are induced by the ones of $P$, which is limiting and does not lead to interpretable and structured representations. On the other hand, some recently introduced frameworks aim to jointly learn the equivariant representation together with the latent dynamics/symmetries (\cite{kipf2019contrastive, van2020plannable}). Although this has the advantage that the group of symmetries is not assumed to be known a priori, the obtained representation is again unstructured, uninterpretable, and comes with no theoretical guarantees. \\

 We empirically investigate our framework on image datasets with a variety of symmetries including translations, dilations and rotations. We moreover provide both qualitative and quantitative comparisons with competing equivariant representation learning frameworks. Results show that our representations exhibit more structure and outperform the baselines in terms of latent symmetry prediction. Moreover, we show how the preservation of geometry of our framework can be applied to a \emph{mapping} task: for data collected by a mobile agent our representation can be used to extract maps of multiple environments simultaneously. We provide a Python implementation together with data and code for all the experiments at a publicly available repository \footnote{\small	\url{https://github.com/equivariant-ml/equivariant-representation-learning}}. In summary, our contributions include:
\begin{itemize}
\item  A method for learning equivariant representations separating intrinsic data classes from poses. 

\item A general mathematical formalism based on group theory, which ideally guarantees lossless and disentangled representations. 

\item An empirical investigation via a set of experiments involving various group actions, together with applications to scene mapping through visual data. 
\end{itemize}

%% file: sections/method.tex
\section{THE MATHEMATICS OF SYMMETRIES}\label{group}

We now introduce the necessary mathematical background on symmetries and equivariance. The modern axiomatization of symmetries relies on their algebraic structure i.e., composition and inversion. The properties of those operations are captured by the abstract concept of a \emph{group} (\cite{rotman2012introduction}). 

\begin{defn}\label{groupdef}
A group is a set $G$ equipped with a \emph{composition map} $G \times G \rightarrow G$ denoted by $(g,h) \mapsto gh$, an \emph{inversion map} $G \rightarrow G$ denoted by $g \mapsto g^{-1}$, and a distinguished \emph{identity element} $1 \in G$ such that for all $g, h, k \in G$:
\begin{center}
\begin{tabular}{ccc}
\emph{Associativity} &   \emph{Inversion} & \emph{Identity}  \\
 $g(hk) = (gh)k$ & \hspace{.2cm} $g^{-1}g = g g^{-1} = 1$  \hspace{.2cm} & $g1 = 1g = g$ 
\end{tabular}
\end{center}
\end{defn}

Examples of groups include the permutations of a set and the general linear group $\textnormal{GL}(n)$ of $n \times n$ invertible real matrices, both equipped with usual composition and inversion of functions. An interesting subgroup of the latter is the special orthogonal group $\textnormal{SO}(n) = \{ A \in \textnormal{GL}(n) \ | \ A A^T = 1, \ \textnormal{det}(A)= 1 \}$, which consists of linear orientation-preserving isometries of the Euclidean space. An example of a commutative group (i.e., such that $gh = hg$ for all $g,h \in G$) is $\mathbb{R}^n$ equipped with vector sum as composition. 

The idea of a space $\mathcal{X}$ having $G$ as a group of symmetries is abstracted by the notion of group \emph{action}. 
\begin{defn}\label{actiondef}
An action by a group $G$ on a set $\mathcal{X}$ is a map $G \times \mathcal{X} \rightarrow \mathcal{X}$ denoted by $(g,x) \mapsto g \cdot x$, satisfying for all $g,h \in G, \ x \in \mathcal{X}$: 
\begin{center}
\begin{tabular}{ccc}
\emph{Associativity} & \hspace{1cm} & \emph{Identity} \\
$g\cdot (h \cdot x) = (gh) \cdot x$ & \hspace{1cm} &  $1 \cdot x = x$
 \end{tabular}
  \end{center}
\end{defn}

In general, the following actions can be defined for arbitrary groups: $G$ acts on any set \emph{trivially} by $g \cdot x = x$, and $G$ acts on itself seen as a set via (left) \emph{multiplication} by  $g \cdot h = gh$. A further example of group action is $\textnormal{GL}(n)$ acting on $\mathbb{R}^n$ by matrix multiplication.

Maps which preserve symmetries are called \emph{equivariant} and will constitute the fundamental notion of our representation learning framework. 

\begin{defn}
A map $\varphi \colon \ \mathcal{X} \rightarrow \mathcal{Z}$ between sets acted upon by $G$ is called \emph{equivariant} if $\varphi(g \cdot x) = g \cdot \varphi(x)$ for all $g \in G, x\in \mathcal{X}$. It is called \emph{invariant} if moreover $G$ acts trivially on $\mathcal{Z}$ or, explicitly, if  $\varphi(g \cdot x) = \varphi(x)$. It is called \emph{isomorphism} if it is bijective.
\end{defn}

 Now, group actions induce classes in $\mathcal{X}$ called \emph{orbits} by identifying points related by a symmetry. 
\begin{defn}
Consider the equivalence relation on $\mathcal{X}$ given by deeming $x$ and $y$ equivalent if $y = g \cdot x$ for some $g \in G$. The induced equivalence classes are called \emph{orbits}, and the set of orbits is denoted by $\mathcal{X} / G$. 
\end{defn}
\begin{figure}
\centering
\includegraphics[width=.95\linewidth]{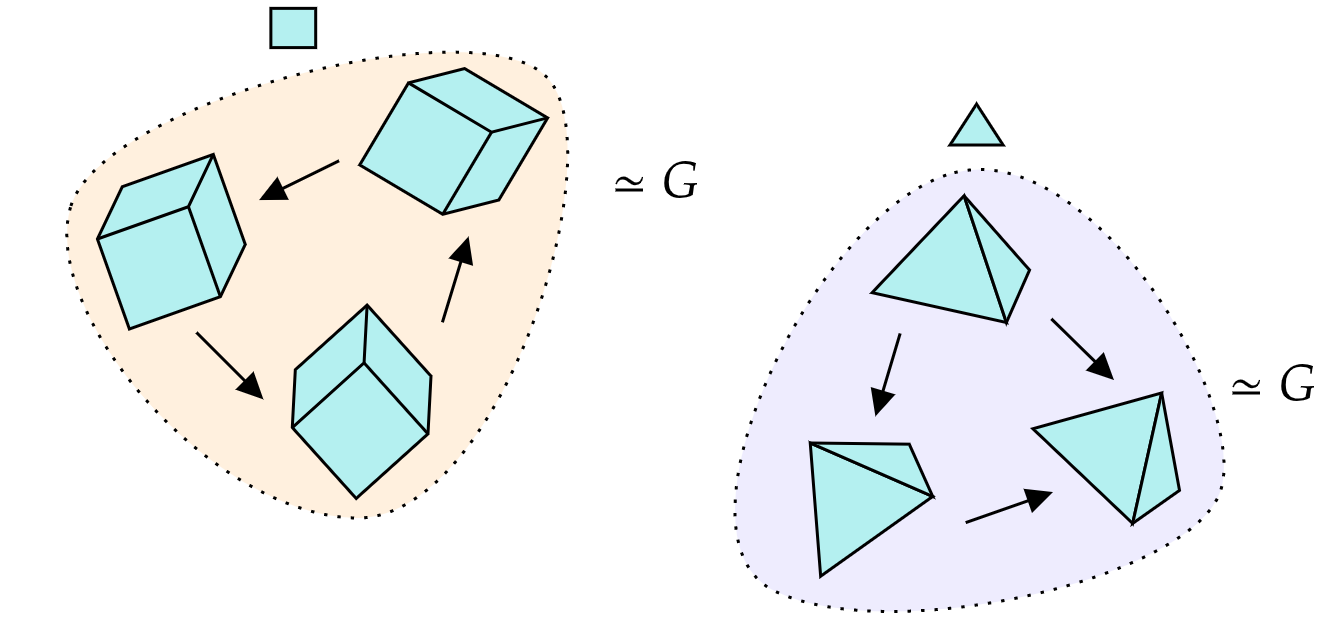}
\caption{Orbits of a (free) group action represent intrinsic classes of data. Each orbit is isomorphic to the symmetry group $G$ itself.}
\label{figorbits}
\end{figure}
For example the orbits of the trivial action are singletons, while the multiplication action has a single orbit. Data-theoretically, an orbit may be seen as an invariant, maximal class of data induced by the symmetry structure. In the example of rigid objects acted upon by translations and rotations, orbits indeed correspond to shapes (see Figure \ref{figorbits}).

It is intuitive to assume that a nontrivial symmetry $g \not = 1 \in G$ has to produce a change in data. If no difference is perceived, one might indeed consider the given transformation as trivial. We can thus assume that no point in $\mathcal{X}$ is fixed by an element of $G$ different from the identity or, in other words, $g \cdot x \not = x$ for $g \not = 1$. Such actions are deemed as \emph{free} and will be the ones relevant to the present work.  

\begin{assm}
\label{eq:free_action}
The action by $G$ on $\mathcal{X}$ is free.
\end{assm}

The following is the core theoretical result motivating our representation learning framework, which we will discuss in the following section. The result guarantees a general decomposition into a trivial and a multiplicative action and describes all the equivariant isomorphisms of such a decomposition.

\begin{prop}\label{iso}
The following holds: 

\begin{itemize}
    \item There is an equivariant isomorphism
\begin{equation}
\mathcal{X} \simeq (\mathcal{X} / G ) \times G
\end{equation}
where $G$ acts trivially on the orbits  and via multiplication on itself, i.e., $g \cdot (e, h) = (e, gh)$ for $g,h \in G$, $e \in \mathcal{X} / G$. In other words, each orbit can be identified equivariantly with the group itself. 

\item Any equivariant map $\varphi :  \ (\mathcal{X} / G ) \times G \rightarrow (\mathcal{X} / G ) \times G$ is a right multiplication on each orbit i.e., for each orbit $O \in \mathcal{X}/G$ there is an $h_O \in G$ such that $\varphi(O, g) = (O', g h_O)$ for all $g \in G$. In particular, if $\varphi$ induces a bijection on orbits then it is an isomorphism. 

\end{itemize}
\end{prop}

We refer to the Appendix for a proof. The first part of the statement can be interpreted in plain words as a \emph{decomposition of classes from poses} for any free group action. According to this terminology, a pose is abstractly an element of an arbitrary group $G$
while a class is an orbit. The intuition behind the second part of the statement is that any equivariant map performs an orbit-dependent `change of frame' in the sense that elements of an orbit $O$ get composed on the right by a symmetry $h_O$ depending on $O$. This will imply that our representations can differ from ground-truth ones only by such change of frames and will in fact guarantee isomorphic representations for our framework.

\section{METHOD}\label{repr}

    \subsection{General Equivariant Representation Learning} 
    In the context of \emph{representation learning} the goal of the model is to learn a map deemed `representation' $\varphi \colon \ \mathcal{X} \rightarrow \mathcal{Z}$ from the data space $\mathcal{X}$ to a \emph{latent space}. The learner optimizes a loss $\mathcal{L} \colon \ \mathcal{M} \rightarrow \mathbb{R}$ over parameters $ \theta \in \mathcal{M}$ of the map $\varphi = \varphi_\theta$. The so-obtained representation can be deployed in downstream applications for an improved performance with respect to operating in the original data space.

%A meaningful representation, as argued, has to be aware of symmetries. Not only do they model the geometry of the perceived space, but are essential in learning and reasoning as well. An intelligent agent capable of applying symmetries on its latent  space might distinguish intrinsic classes (that is, orbits of $G$) and predict changes due to transformations, leveraging on such information more or less directly in order to achieve a task. 

The central assumption of \emph{equivariant} representation learning is that data carries symmetries which the representation has to preserve. As discussed in Section \ref{group}, this means that a group of symmetries $G$ acts on both $\mathcal{X}$ and $\mathcal{Z}$ and that the representation $\varphi$ is encouraged to be equivariant via the loss. While the action on $\mathcal{Z}$ is designed as part of the model, the action on $\mathcal{X}$ is unknown in general and has to be conveyed by data. Concretely, the dataset consists of triples $(x, g, y)$ with $x \in \mathcal{X}, g \in G$ and $y = g \cdot x$. The group element $g$ carries symmetry information which is relative between $x$ and $y$. Equivariance is then naturally encouraged via a loss in the form: 
\begin{equation}\label{loss}
\mathcal{L}(\theta; x,g, y) = d\left(\varphi_\theta(y), \ g \cdot \varphi_\theta(x) \right).
\end{equation}
Here $d \colon \ \mathcal{Z} \times \mathcal{Z} \rightarrow \mathbb{R}_{\geq 0}$ is a similarity function on $\mathcal{Z}$. Generally speaking, $d$ does not necessarily need to satisfy the axioms for a distance but we at least require it to be \emph{positive definite} i.e., $d(z,z') = 0$ iff $z=z'$. Note that we assume the group together with its algebraic structure to be known a priori and not inferred during the learning process. Its action over the latent space is defined in advance and constitutes the primary inductive bias for equivariant representation learning.

\subsection{Learning to Decompose Class and Pose} Motivated by Proposition \ref{iso}, we propose to set the latent space as: 
\begin{equation}\label{latentsep}
\mathcal{Z} = \underbrace{\mathcal{E}}_{Class} \times \underbrace{G}_{Pose}
\end{equation}
with $G$ acting trivially on $\mathcal{E}$ and via multiplication on itself. Here, $\mathcal{E}$ is any set which is meant to represent classes of the encoded data. Since there is in general no prior information about the action by symmetries on $\mathcal{X}$ and its orbits, $\mathcal{E}$ has to be set beforehand. Assuming $\mathcal{E}$ has enough capacity to contain $\mathcal{X} / G$, Proposition \ref{iso} shows that an isomorphic equivariant data representation is possible in $\mathcal{Z}$. By fixing (positive definite) similarity functions $d_\mathcal{E}$ and $d_G$ on $\mathcal{E}$ and $G$ respectively, we obtain a joint latent similarity function $d(z, z') =  d_\mathcal{E}(z_\mathcal{E},z'_\mathcal{E}) +  d_G(z_G,z'_G)$, where the subscripts denote the corresponding components. When $G \subseteq \textnormal{GL}(n)$ is a group of matrices, a typical choice for $d_G$ is the (squared) Frobenius distance i.e., the Euclidean distance for matrices seen as flattened vectors. The equivariance loss $\mathcal{L}(\theta; x,g, y)$ in Equation \ref{loss} then reads:
% \begin{align}\label{equivloss}
%     \begin{split}
%     \mathcal{L}(\theta; x,g, y) =   &\underbrace{ d_\mathcal{E}\left(\varphi^\mathcal{E}(y), \ \varphi^\mathcal{E}(x)\right)}_{Invariant} \ + \\ 
%     + &\underbrace{d_G\left(\varphi^G(y), \ g\varphi^G(x)\right)}_{Multiplication-Equivariant}
%     \end{split}
% \end{align}
\begin{equation}\label{equivloss}
    \underbrace{ d_\mathcal{E}\left(\varphi^\mathcal{E}(y), \ \varphi^\mathcal{E}(x)\right)}_{Invariant} \ + \underbrace{d_G\left(\varphi^G(y), \ g\varphi^G(x)\right).}_{Multiplication-Equivariant}
\end{equation}

   Here we denoted the components of the representation map by $\varphi = (\varphi^\mathcal{E}, \varphi^G)$ and omitted the parameter $\theta$ for simplicity. To spell things out,  $\varphi^\mathcal{E}$ encourages data from the same orbit to lie close in $\mathcal{E}$  (i.e., $\varphi^\mathcal{E}$ is ideally invariant) while $\varphi^{G}$ aims for equivariance with respect to multiplication on the pose component $G$. 
    
    If $\varphi_{\mathcal{E}}$ is injective then Proposition \ref{iso} guarantees lossless (i.e., isomorphic) representations, which we summarize in the following corollary: 
    
    \begin{cor}
    Suppose that $\varphi_{\mathcal{E}}$ is injective. Then $\mathcal{L}(\theta; x,g, y)  = 0$ for all $x,g,y=g \cdot x$ if and only if $\varphi_\theta$ is an equivariant isomorphism on its image. 
    \end{cor}
    In order to force injectivity, we propose a typical solution from contrastive learning literature (\cite{chen2020simple}) encouraging latent points to spread apart. To this end, we opt for the standard \emph{InfoNCE loss} (\cite{oord2018representation}), although other choices are possible. This means that we replace the term $d_\mathcal{E}\left(\varphi^\mathcal{E}(y), \ \varphi^\mathcal{E}(x)\right)$
    in Equation \ref{equivloss} with \begin{equation}
        \label{infonceeq}
        \frac{1}{\tau}d_\mathcal{E}\left(\varphi^\mathcal{E}(y), \ \varphi^\mathcal{E}(x)\right) +  \log \mathbb{E}_{x'} \left[ e^{-\frac{1}{\tau} d_\mathcal{E}\left(\varphi^\mathcal{E}(x'), \ \varphi^\mathcal{E}(x)\right)} \right].
    \end{equation}
The hyperparameter $\tau \in \mathbb{R}_{>0}$ (`temperature') controls the amount of latent entropy. Following \cite{oord2018representation, chen2020simple}, we set the class component as a sphere $\mathcal{E}= \mathbb{S}^m$ by normalizing the output of $\varphi_\mathcal{E}$. This allows to deploy the cosine dissimilarity $d_{\mathcal{E}}(z,z') = - \textnormal{cos}( \angle z z') = - z \cdot z'$ and is known to lead to improved performances due to the compactness of $\mathcal{E}$ (\cite{wang2020understanding}).

\subsection{Parametrizing via the Exponential Map}
The output space of usual machine learning models such as deep neural networks is Euclidean. Our latent space (Equation \ref{latentsep}) contains $G$ as a factor, which might be non-Euclidean as in the case of $G=\textnormal{SO}(n)$. In order to implement our representation learner $\varphi$ it is thus necessary to parametrize the group $G$. To this end, we assume that $G$ is a differentiable manifold (with differentiable composition and inversion maps) i.e., that $G$ is a \emph{Lie group}. One can then define the \emph{Lie algebra} $\mathfrak{g}$ of $G$ as the tangent space to $G$ at $1$. 

%The Lie algebra is equipped with a skew-symmetric product $\mathfrak{g} \times \mathfrak{g} \rightarrow \mathfrak{g}$ which is denoted by $(v,w) \mapsto [v,w]$ and can be succinctly defined as the differential of the commutator map $(g,h) \mapsto ghg^{-1}h^{-1}$. The Lie product of $G = \mathbb{R}^n$, for example, vanishes because of commutativity. When $G \subseteq \textnormal{GL}(n)$, $\mathfrak{g}$ is contained in the space of $n \times n$ matrices and the Lie product amounts to $[A,B] = AB -BA$. For $ G = \textnormal{SO}(3)$, $\mathfrak{g}$ consists of $3 \times 3$ skew-symmetric matrices and the Lie product coincides with the vector product in $\mathbb{R}^3 \simeq \mathfrak{g}$. 

We propose to rely on the \emph{exponential map} $\mathfrak{g} \rightarrow G$, denoted by $v \mapsto e^v$, to parametrize $G$. This means that $\varphi$ outputs an element $v$ of $\mathfrak{g}$ that gets mapped into $G$ as $e^v$. Although the exponential map can be defined for general Lie groups by solving an appropriate ordinary differential equation, we focus on the case $G \subseteq \textnormal{GL}(n)$. The Lie algebra $\mathfrak{g}$ is then contained in the space of $n \times n$ matrices and the exponential map amounts to the matrix Taylor expansion $e^v = \sum_{k \geq 0} v^k / k! $. For specific groups the latter can be simplified via simple closed formulas. For example, the exponential map of $\mathbb{R}^n$ is the identity while for $\textnormal{SO}(3)$ it can be efficiently computed via the Rodrigues' formula (\cite{liang2018efficient}).

\subsection{Relation to Disentanglement}\label{disent}
Our equivariant representation learning framework is related to the popular notion of \emph{disentanglement} (\cite{bengio_disentanglement, bvae}). Intuitively, in a disentangled representation a variation of a distinguished aspect in the data is reflected by a change of a single component in the latent space. Although there is no common agreement on a rigorous formulation of the notion (\cite{challenging_disentanglement}), a proposal has been addressed in \cite{higgins2018towards}. The presence of multiple dynamic aspects in the data is formalized as an action on $\mathcal{X}$ by a decomposed group
\begin{equation}\label{factor}
G = G_1 \times \cdots \times G_n 
\end{equation}
where each of the factors $G_i$ is responsible for the variation of a single aspect. A representation $\varphi \colon \ \mathcal{X} \rightarrow \mathcal{Z}$ is then defined to be disentangled if (i) there is a decomposition $ \mathcal{Z} = \mathcal{Z}_1 \times \cdots \times \mathcal{Z}_n $ where each $\mathcal{Z}_i$ is acted upon trivially by the factors $G_j$ with $j \not = i$ and (ii) $\varphi$ is equivariant. 

Our latent space (Equation \ref{latentsep}) automatically yields to disentanglement in this sense. Indeed, in the case of a group as in Equation \ref{factor} we set  $\mathcal{Z}_i = G_i$. In order to deal with the remaining factor $\mathcal{Z}_0 = \mathcal{E}$, a copy of the trivial group $G_0 = \{ 1 \}$ can be added to $G$ without altering it up to isomorphism. The group $G \simeq G_0 \times \cdots \times G_n$ acts on $\mathcal{Z} = \mathcal{E}\times G = \mathcal{Z}_0 \times \cdots \times \mathcal{Z}_n$ as required for a disentangled latent space. In conclusion, our work fits in the line of research aiming to infer disentangled representation via indirect and weak forms of supervision (\cite{locatello2020weakly}), of which symmetry structures are an example.

%% file: sections/related_work.tex
\section{RELATED WORK}\label{related}
%\subsection{Equivariant Representation Learning}

\textbf{Equivariant Representation Learning}. Models relying on symmetry and equivariance have been studied in the context of representation learning. These models are typically trained on variations of the equivariance loss (Equation \ref{loss}) and are designed for specific groups $G$ and actions on the latent space $\mathcal{Z}$. The pioneering \emph{Transforming Autoencoders} (\cite{hinton2011transforming}) learn to represent image data translated by  $G = \mathbb{R}^2$ in the pixel plane, with $\mathcal{Z}$ consisting of several copies of $G$ (`capsules') acting on itself. Although such models are capable of learning isomorphic representations, the orbits are not explicitly modeled in the latent space. In contrast, our invariant component $\mathcal{E}$ is an interpretable alternative to multiple capsules making orbits recoverable from the representation. A series of other works represent data in a latent space modelled via the group of symmetries $G$ (\cite{homeomorphic, tonnaer2020quantifying, yang2021towards}). These works however either do not reserve additional components dedicated to orbits, obtaining a representation that forgets the intrinsic classes of data, or address specific groups i.e., $\textnormal{SO}(3)$, the torus $\textnormal{SO}(2) \times \textnormal{SO}(2)$ and product of cyclic groups respectively. \emph{Affine Equivariant Autoencoders} (\cite{guo2019affine}) deal with affine transformations of the pixel-plane (shearing an image, for example) and implement a latent action through a hand-crafted map $t \colon \ G \rightarrow \mathcal{Z} = \mathbb{R}^n$.  Groups of rotations $\textnormal{SO}(n)$ linearly acting on a Euclidean latent space $\mathcal{Z}=\mathbb{R}^n$ are explored in  \cite{worrall2017interpretable, dynamic_enviroments}. Since rotating a vector around itself has no effect, linear actions are not free (for $n \geq 3$), which makes isomorphic representations impossible. \emph{Equivariant Neural Rendering} (\cite{renderer}) proposes a latent voxel grid on which $\textnormal{SO}(3)$ acts approximately by rotating and interpolating values. In contrast, our latent group action is exact and thus induces no loss of information. We provide an empirical comparison to both linear Euclidean actions and Equivariant Neural Rendering in Section \ref{experiments}. Lastly, \cite{winter2022unsupervised} have recently proposed to learn equivariant representations by splitting the latent space into an invariant component and an equivariant one, which bears similarity to our framework. Differently from us, however, the model is trained via an equivariance loss in the data space $\mathcal{X}$ and thus requires the group actions over data to be known a priori. As previously discussed, this is limiting and often unrealistic in practice.  

% \subsection{Convolutional Networks}
\textbf{Convolutional Networks}. Convolutional layers in neural networks (\cite{taco_group_equiv, cohengeneraltheory}) satisfy equivariance a priori with respect to transformations of the pixel plane. They were originally introduced for (discretized) translations and later extended to more general groups (\cite{sphericalcnn, symmetricsets, cohen2019gauge}). However, they require data and group actions in a specific form. Abstracty speaking, data need to consist of vector fields over a base space (images seen as RGB fields over the pixel plane, for example) acted upon by $G$, which does not hold in general. Examples of symmetries not in this form are changes in perspective of first-person images of a scene and  rotations of rigid objects on an image. Our model is instead applicable to arbitrary (Lie) group actions and infers equivariance in a data-driven manner. Moreover, equivariance through $G$-convolutions alone is hardly suitable for representation learning as the output dimension coincides with the input one. Dimensionality reduction techniques deployed together with convolutions such as max-pooling or fully-connected layers disrupt equivariance completely. The latent space in our framework is instead compressed and is ideally isomorphic to the data space $\mathcal{X}$ (Proposition~\ref{iso}).  

\begin{center}

\begin{table*}
\setlength\extrarowheight{5pt}
\caption{Datasets involved in our experiments, with the corresponding group of symmetries and number of orbits.}
\label{datasettable}
\centering
\begin{tabu}{cccccc} 
\toprule
  & \textsc{Sprites} & \textsc{Shapes} & \textsc{Multi-Sprites} & \textsc{Chairs} & \textsc{Apartments}      \\
    \midrule
    \midrule 
 $\mathcal{X}$ &   \begin{minipage}{.15\textwidth}
       \vspace{0.1cm} 
      \includegraphics[width=\linewidth]{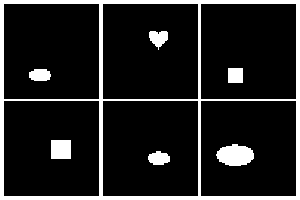}  \end{minipage}  \vspace{0.1cm} 
     
    & \begin{minipage}{.15\textwidth}
       \vspace{0.1cm}\includegraphics[width=\linewidth]{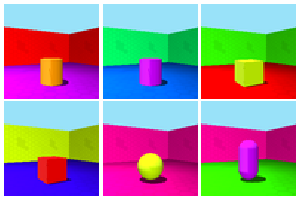}
    \end{minipage}   
    &  \begin{minipage}{.15\textwidth}
     \vspace{0.1cm} \includegraphics[width=\linewidth]{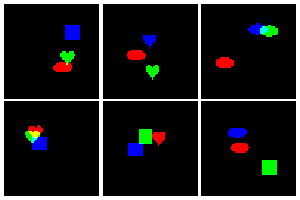}
    \end{minipage}  
    &
    \begin{minipage}{.15\textwidth}
       \vspace{0.1cm} \includegraphics[width=\linewidth]{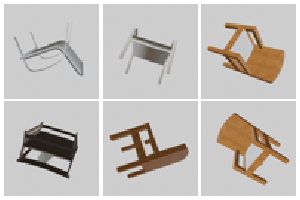}
    \end{minipage}
    & \begin{minipage}{.15\textwidth}
       \vspace{0.1cm} \includegraphics[width=\linewidth]{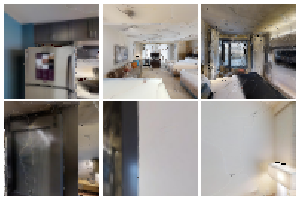}
    \end{minipage}    \\
  \midrule
  $G$ & $\mathbb{R}^3$ & $\mathbb{R}^3$ & $\mathbb{R}^6$ & $\textnormal{SO}(3)$ &   $\mathbb{R}^2 \times \textnormal{SO}(2)$   
\\
  \midrule 
$ |  \mathcal{X} / G | \  $ & $3$ & $4$ & $27$ & $3$ & $2$ \\
\bottomrule
\end{tabu}
 \end{table*}
\end{center}

% \subsection{World Models}
\textbf{World Models}. Analogously to group actions, Markov Decision Processes (MDPs) from reinforcement learning and control theory involve a possibly stochastic interaction $\mathcal{A}\times\mathcal{X} \rightarrow \mathcal{X}$ with an environment $\mathcal{X}$ via a set $\mathcal{A}$ of moves. In general, no algrbraic structure (such as a group composition) is assumed on $\mathcal{A}$. In this context, a representation equivariant with respect to the action is referred to as \emph{World Model} (\cite{ha2018world, kipf2019contrastive, park2022learning}) or \emph{Markov Decision Process Homomorphism} (MDPH) (\cite{van2020plannable}). MDPHs are usually deployed as pre-trained representations for downstream tasks or trained jointly with the agent for exploration purposes (\cite{curiosity}). However, the latent action $\mathcal{A} \times \mathcal{Z} \rightarrow \mathcal{Z}$ of an MDPH is learned since no prior knowledge is assumed around $\mathcal{A}$ or the environment. This implies that the resulting representation is unstructured and uninterpretable. We instead assume that $G=\mathcal{A}$ is a group acting (freely) on $\mathcal{X}$, which enables us to define a geometrically-structured and disentangled latent space that guarantees isomorphic equivariant representations. We provide an empirical comparison to MDPHs in Section \ref{experiments}.

%% file: sections/experiments.tex
% {
% \extrarowsep=2ex
% \begin{center}
% \begin{table}
% \caption{Summary of the datasets involved in our experiments, together with the corresponding group of symmetries and number of orbits.}
% \centering
% \begin{tabu}{c|c|c|c|c|c} 
%   & Sprites & Shapes & Multi-Sprites & Chairs & Apartments  \\ 
%   \hline
%     \hline
%   $\mathcal{X}$ &   \begin{minipage}{.17\textwidth}
%       \includegraphics[width=\linewidth]{figures/random_sprites.png}
%     \end{minipage} 
    
%     & \begin{minipage}{.17\textwidth}
%       \includegraphics[width=\linewidth]{figures/random_shift.png}
%     \end{minipage}  
%     &  \begin{minipage}{.17\textwidth}
%       \includegraphics[width=\linewidth]{figures/random_multi.png}
%     \end{minipage}  
%     &
%     \begin{minipage}{.17\textwidth}
%       \includegraphics[width=\linewidth]{figures/random_chairs.png}
%     \end{minipage}
%     & \begin{minipage}{.17\textwidth}
%       \includegraphics[width=\linewidth]{figures/random_rooms.png}
%     \end{minipage}  \\
%   \hline
%  $G$ & $\mathbb{R}^3$ & $\mathbb{R}^3$ & $\mathbb{R}^6$ & $\textnormal{SO}(3)$ &   \makecell{$\mathbb{R}^2 \times \textnormal{SO}(2)$ \\ $\mathbb{R}^2 \rtimes \textnormal{SO}(2)$} \\
%   \hline
% $ |  \mathcal{X} / G | \  $ & $3$ & $4$ & $27$ & $3$ & $2$
% \end{tabu}
%     \end{table}
% \end{center}
% }

\section{EXPERIMENTS}\label{experiments}

\subsection{Dataset Description}
Our empirical investigation aims to assess our framework via both qualitative and quantitative analysis on datasets with a variety of symmetries. To this end we deploy five datasets summarized in Table \ref{datasettable}: three with translational symmetry extracted from dSprites and 3DShapes (\cite{sprites, shapes}), one with rotational symmetry extracted from ShapeNet (\cite{chang2015shapenet}) and one simulating a mobile agent exploring apartments and collecting first-person views. The latter is extracted from Gibson (\cite{xiazamirhe2018gibsonenv}) and generated via the Habitat simulator (\cite{habitat19iccv}). Datapoints are triples $(x,g,y)$ where $x,y$ are $64 \times 64$ images, $g\in G$ and $y = g \cdot x$. We refer to the Appendix for a more detailed description of the datasets. 

\subsection{Baselines and Implementation Details}\label{sec:baselines}
We compare our method with the following models designed for learning equivariant representations: 

\textbf{MDP Homomorphisms} (MDPH) from \cite{van2020plannable}, \cite{kipf2019contrastive}: a framework where the representation $\varphi: \ \mathcal{X} \rightarrow \mathcal{Z}$ is learnt jointly with the latent action $T: \ G \times \mathcal{Z} \rightarrow \mathcal{Z}$. The two models are trained with the equivariance loss $\mathbb{E}_{x,g,y=g \cdot x}[d(\varphi(y), \ T(g, \varphi(x)))]$ (cf. Equation \ref{loss}). In order to avoid trivial solutions, an additional `hinge' loss term $\mathbb{E}_{x,x'}[\max \{ 0, \ 1 -  d(x,x')\}]$ is optimized that encourages encodings to spread apart in the latent space. This is analogous to (the denominator of) the InfoNCE loss (Equation \ref{infonceeq}) which we rely upon to avoid orbit collapse in $\mathcal{E}$. Differently from us, an MDPH does not assume any prior knowledge on $G$ nor any algebraic structure on the latter. However, this comes at the cost of training an additional model $T$ and losing the structures and guarantees provided by our framework. 

\textbf{Linear}: a model with $\mathcal{Z} = \mathbb{R}^3$ on which $\textnormal{SO}(3)$ acts by matrix multiplication. Such a latent space has been employed in previous works (\cite{worrall2017interpretable, dynamic_enviroments}). The model is trained with the same loss as MDPH i.e., equivariance loss together with the additional hinge term avoiding collapses such as $\varphi = 0$. Note that the action on $\mathcal{Z}$ is no longer free (even away from $0$) since rotating a vector around itself has no effect. Differently from our method, the model is thus forced to lose information in order to learn an equivariant representation. 

\textbf{Equivariant Neural Renderer} (ENR) from \cite{renderer}: a model with a tensorial latent space $\mathcal{Z} = \mathbb{R}^{C \times D \times H \times W}$, thought as a copy of $\mathbb{R}^C$ for each point in a $D \times H \times W$ grid in $\mathbb{R}^3$. The group $\textnormal{SO}(3)$ \emph{approximately} acts on $\mathcal{Z}$ by rotating the grid and interpolating the obtained values in $\mathbb{R}^C$. The model is trained trained jointly with a decoder $\psi: \ \mathcal{Z} \rightarrow \mathcal{X}$ and optimizes variation of the equivariance loss incorporating reconstruction: $\mathbb{E}_{x,g, y=g\cdot x}[d_\mathcal{X}(y, \ \psi(g \cdot \varphi(x))   )]$. We set $d_\mathcal{X}$ as the standard binary cross-entropy metric for (normalized) images. Although the action on $\mathcal{Z}$ is free, the latent discretisation and consequent interpolation make the model only approximately equivariant.

\begin{figure*}[t]
\centering
\includegraphics[width=1.\linewidth]{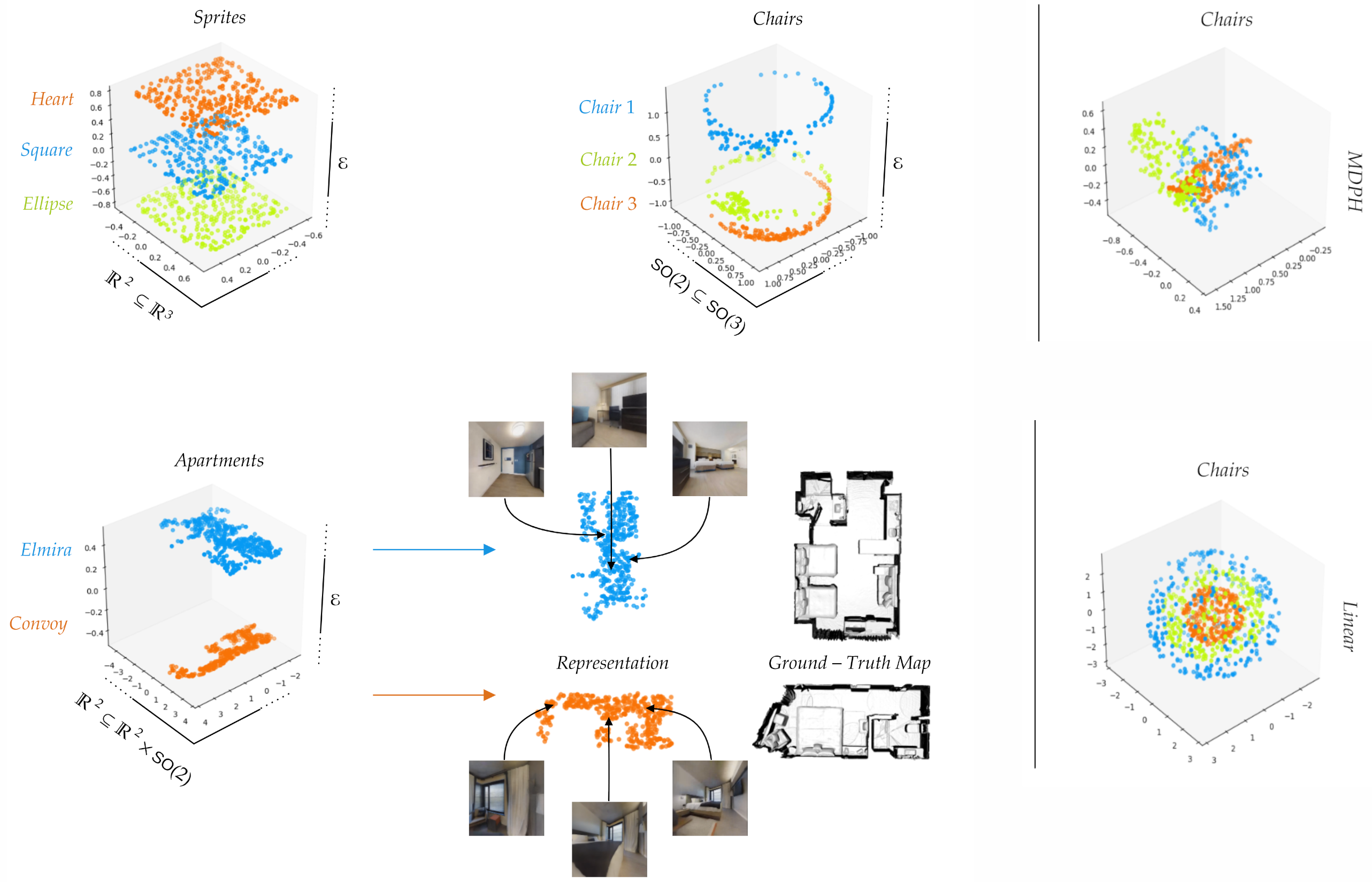}
\caption{\textbf{Left:} visualization of encodings through $\varphi$ from the Sprites, Chairs and Apartments datasets. The images display the projection to the annotated components of $\mathcal{Z}$ and data are colored by their ground-truth class. Each latent orbit from Apartments is compared to the view from the top of the corresponding scene. \textbf{Right:} same visualization for the baseline models MDPH and Linear on the Chairs dataset. }
\label{fig:visualizations} 
\end{figure*}

We implement the equivariant representation learner $\varphi$ as a ResNet-$18$ (\cite{resnet}), which is a deep convolutional neural network with residual connections. We train our models for $100$ epochs through stochastic gradient descent by means of the Adam optimizer with learning rate $10^{-3}$ and batch size $16$. The distance $d_G$ is set as the squared Euclidean one for $G=\mathbb{R}^n$ and for $G = \textnormal{SO}(2) \subseteq \mathbb{R}^2$, while it is set to the squared Frobenius one for  $G=\textnormal{SO}(3)$. The invariant component consists of a sphere $\mathcal{E} = \mathbb{S}^7 \subseteq \mathbb{R}^8$ (see Section \ref{repr}) parametrized by the normalized output of $8$ neurons in the last layer of $\varphi$. All the models  from Section \ref{sec:baselines} implement the same architecture (ResNet-$18$). ENR moreover implements $3$D convolutional layers around the latent space as suggested in the original work (\cite{renderer}). The latent action model $T$ for MDPH is implemented as a two-layer deep neural network ($128$ neurons per layer) with ReLU activation functions. For MDPH we set $\textnormal{dim}(\mathcal{Z}) = 8 + \textnormal{dim}(G)$, which coincides with the output dimensionality of our model.

% The machine we run the experiments on possesses an AMD Ryzen $5$ CPU, an Nvidia GeForce GTX $2070$ Super GPU and $32$GB of RAM. We implement our models in Python$3$ with the PyTorch framework. The architecture and hyperparameters of the learner $\varphi$ are the ones reported in \cite{resnet} for ResNet$18$.
\begin{table*}[h]
\centering
\setlength{\tabcolsep}{12pt}
\setlength{\extrarowheight}{4pt}
\caption{Hit-rate (mean and std over $3$ runs) on test trajectories of increasing length.}
\label{tablebold}
\begin{tabular}{ccccc}
\toprule 
  %  & & \multicolumn{2}{c}{1 Step} & \multicolumn{2}{c}{10 Steps} & \multicolumn{2}{c}{20 Steps}\\
    Dataset   & Model  & 1 Step  & 10 Steps   & 20 Steps
        \\
    \midrule
    \midrule 
    \multirow{2}{*}{\textsc{Sprites}} %\rotatebox[origin=c]{90}{\makecell{Sprites}}}
    %\midrule
    & Ours     & ${\bf 1.00}_{\pm 0.00}$      & ${\bf 1.00}_{\pm 0.00}$   & ${\bf 1.00}_{\pm 0.00}$  \\
    & MDPH     & ${\bf 1.00}_{\pm 0.00}$      &   ${\bf 1.00}_{\pm 0.00}$  & $0.98_{\pm 0.02}$ \\
    %\vspace{0.1cm} \\

    \midrule
    \multirow{2}{*}{\textsc{Shapes}}
    %\rotatebox[origin=c]{90}{\makecell{Shapes}}}
    %\midrule
    & Ours     & ${\bf 1.00}_{\pm 0.00}$   &   ${\bf 1.00}_{\pm 0.00}$  & ${\bf 1.00}_{\pm 0.00}$  \\
    & MDPH   & ${\bf 1.00}_{\pm 0.00}$   &  $0.99_{\pm 0.01}$  & $0.96_{\pm 0.04}$ \\
    %\vspace{0.1cm} \\
    \midrule 
    \multirow{2}{*}{\textsc{Multi-Sprites}}
    %\rotatebox[origin=c]{90}{\makecell{Multi-\\Sprites}}}
    %\midrule
    & Ours   & ${\bf 1.00}_{\pm 0.00}$  &   ${\bf 0.93}_{\pm 0.03}$ & ${\bf 0.93}_{\pm 0.03}$ \\
    & MDPH     & ${\bf 1.00}_{\pm 0.00}$    &   $0.28_{\pm 0.06}$  & $0.11_{\pm 0.01}$  \\
    %\vspace{0.1cm} \\
    \midrule
    \multirow{4}{*}{\textsc{Chairs}}
    %\rotatebox[origin=c]{90}{\makecell{Chairs}}}
    %\midrule
    & Ours     & ${\bf 0.98}_{\pm 0.01}$     &  ${\bf 0.94}_{\pm 0.01}$  & ${\bf 0.94}_{\pm 0.01}$  \\
    & Linear     & $0.89_{\pm 0.08}$    &   $0.87_{\pm 0.10}$  & $0.87_{\pm 0.10}$\ \\
    & MDPH    & ${\bf 0.98}_{\pm 0.00}$   &   $0.88_{\pm 0.07}$ &  $0.78_{\pm 0.13}$ \\
    & ENR     & ${\bf 0.98}_{\pm 0.00}$   &   $0.91_{\pm 0.01}$  &  $0.82_{\pm 0.05}$  \\
    %\\
    \midrule
    \multirow{2}{*}{\textsc{Apartments}}%{\rotatebox[origin=c]{90}{\makecell{Apart.s}}}
    %\midrule
    & Ours     & ${\bf 0.99}_{\pm 0.00}$       &   ${\bf 0.99}_{\pm 0.00}$   & ${\bf 0.99}_{\pm 0.00}$  \\
    & MDPH     & $0.98_{\pm 0.02}$     &  $0.94_{\pm 0.03}$   &  $0.86_{\pm 0.05}$ \\
  % \hline 
    \bottomrule
\end{tabular}
\end{table*}

\subsection{Visualizations of the Representation}
In this section we present visualizations of the latent space of our model (Equation \ref{latentsep}), showcasing its geometric benefits. The preservation of symmetries coming from equivariance enables indeed to transfer the intrinsic geometry of data explicitly to the representation. Moreover, the invariant component $\mathcal{E}$ separates the orbits of the group action, allowing to distinguish the intrinsic classes of data in the latent space. Finally, the representation from our model automatically disentangles factors of the group as discussed in Section \ref{disent}. 

Figure \ref{fig:visualizations} (left) presents visualizations of encodings through $\varphi$ for the datasets Sprites, Chairs and Apartments. For each dataset we display the projection to one component of $\mathcal{E}$ as well as a relevant component of the group $G$. Specifically, for Sprites we display the component $\mathbb{R}^2 \subseteq G = \mathbb{R}^3$ corresponding to translations in the pixel plane, for Chairs we display a circle $\textnormal{SO}(2) \subseteq G=\textnormal{SO}(3)$ corresponding to one Euler angle while for Apartments we display the component $\mathbb{R}^2 \subseteq G = \mathbb{R}^2 \times \textnormal{SO}(2)$ corresponding to translations in the physical world. For Apartments, we additionally compare representation of    each of the two apartments with the ground-truth view from the top. 

As can be seen, in all cases the model correctly separates the orbits in $\mathcal{E}$ through self-supervision alone. Since the orbits are isomorphic to the group $G$ itself, the model moreover preserves the geometry of each orbit separately. For Sprites, this means that (the displayed component of) each orbit is an isometric copy of the pixel-plane, with disentangled horizontal and vertical translations. (Figure \ref{fig:visualizations}, top-left). For Apartments, this similarly means that each orbit exhibits an isometric copy of the real-world scene. One can recover a map of each of the explored scenes by, for example, estimating the density of data in $\mathcal{Z}$ (Figure \ref{fig:visualizations}, bottom-right) and further use the model $\varphi$ to localize the agent within such map. Our equivariant representation thus solves a \emph{localization and mapping task} in a self-supervised manner and of multiple scenes simultaneously.

As a qualitative comparison, Figure \ref{fig:visualizations} (right) includes visualizations for the models MDPH (trained with $\textnormal{dim}(\mathcal{Z})=3$) and Linear on the Chairs dataset. As can be seen, the latent space of MDPH is unstructured: the geometry of $\mathcal{X}$ is not preserved and classes are not separated. This is because the latent action of MDPH is learned end-to-end and is thus uninterpretable and unconstrained a priori. For Linear the classes are organized as spheres in $\mathcal{Z}$, which are are the orbits of the latent action by $G=\textnormal{SO}(n)$. Such orbits are not isomorphic to $G$ (one Euler angle is missing) since the action is not free. This means that $\mathcal{Z}$ loses information and does not represent the dataset faithfully.   
 
% \begin{center}
% \begin{table}
% \caption{Hit-rate for the models and the datasets considered, with test trajectories of increasing length. The standard deviation is computed over 3 runs.}
% \centering
% \begin{tabular}{c||c|c|c|c||c|c|c|c||c|c|c|c||} 
%  & \multicolumn{4}{c||}{Trajectory Length: 1} & \multicolumn{4}{c||}{Trajectory Length: 10} & \multicolumn{4}{c||}{Trajectory Length: 20} \\
% \hline
%  & MDPH & ENR & Linear & Ours & MDPH & ENR & Linear & Ours  & MDPH & ENR & Linear & Ours  \\
% \hline
%  Sprites &  &  &  & $1.00 \pm 0$ &  &  &  & $1.00 \pm 0$ &  &  &  & $1.00 \pm 0$ \\
%  \hline
%  Shapes &  &  &  &  $1.00 \pm 0$ &  &  &  &  $1.00 \pm 0$ &  &  &  &  $1.00 \pm 0$ \\
%  \hline
%  Multi-Sprites &  &  &  &  &  &  &  & &  &  &  &\\
%  \hline
%  Chairs &  &  &  &   &  &  &  &   &  &  &  & \\
%  \hline
%  Apartments &  &  &  & $0.99 \pm 0$  &  &  &  & $0.99 \pm 0$   &  &  &  & $0.99 \pm 0$  \\

% \end{tabular}
% % \hfill
% % \begin{tabular}{c|c|c} 
% % a & b & c \\
% % \hline
% % a & b & c 
% % \end{tabular}
% \end{table}
% \end{center}

\subsection{Performance Comparison}\label{seccomparison}

In this section we numerically compare our method to the equivariant representation learning frameworks described at the beginning of Section \ref{experiments}. We evaluate the models through \emph{hit-rate}, which is a standard score that allows to compare equivariant representations with different latent space geometries (\cite{kipf2019contrastive}). Given a test triple $(x, g, y=g\cdot x)$, we say that `$x$ hits $y$' if $\varphi(y)$ is the nearest neighbour in $\mathcal{Z}$ of $g \cdot \varphi(x)$ among a random batch of encodings $\{ \varphi(x) \}_{x \in \mathcal{B}}$. For a test set, the hit-rate is then defined as the number of times $x$ hits $y$ divided by the test set size. We set the number of aforementioned random encodings to $|\mathcal{B}|=32$. For each model, the nearest neighbour is computed with respect to the same latent metric $d$ as the one used for training. In order to test the performance of the models when acted upon multiple times in a row, we generate test sets where $g$ is a trajectory i.e., it is factorized as $g = g_1 g_2 \cdots g_T$ for $T \in \{1,10,20 \}$. Hit-rate is then computed after sequentially acting by the $g_i$'s in the latent space. This captures the accumulation of errors in the equivariant representation and thus evaluates the performance for long-term predictions. All the test sets consist of $10 \%$ of the corresponding dataset.

The results are presented in Table \ref{tablebold}. As can be seen, all the models perform nearly perfectly on single-step predictions with the exception of Linear ($89 \%$ hit-rate). For the latter the latent group action is not free, which prevents learning a lossless equivariant representation and thus degrades the quality of predictions. On longer trajectories, however, our model outperforms the baselines by an increasing margin. MDPH accumulates errors due to the lack of structure in its latent space: its latent action is learned, which does not guarantee stability with respect to composition of multiple symmetries. The degradation of performance for MDPH is particularly evident in the case of Multi-Sprites ($11 \%$ hit-rate on $20$ steps), which is probably due to the large number of orbits ($27$) and the consequent complexity of the prediction task. Our model is instead robust even in presence of many orbits ($93 \%$ hit-rate on Multi-Sprites) due to the dedicated invariant component $\mathcal{E}$ in its latent space. 

When the latent space is equipped with a group action, stability on long trajectories follows from \emph{associativity} of the group composition and the action (see Definition \ref{groupdef} and \ref{actiondef}). This is evident from the results for the Chairs dataset, where our model and Linear outperform MDPH on longer trajectories ($94 \%$ and $87 \%$ against $78 \%$ hit-rate on $20$ steps) and exhibit a stable hit-rate as the number of steps increases. Even though ENR carries a latent group action, it still accumulates errors ($82 \%$ hit-rate on $20$ steps) due to the discretization of the its latent space i.e., the latent grid acted upon by $\textnormal{SO}(3)$. The consequent interpolation makes the latent action only approximately associative, causing errors to accumulate on long trajectories. 

\section{CONCLUSIONS, LIMITATIONS AND FUTURE WORK}\label{conc}
In this work we addressed the problem of learning equivariant representations by decomposing the latent space into a group component and an invariant one. We theoretically showed that our representations are lossless, disentangled and preserve the geometry of data. We empirically validated our approach on a variety of groups, compared it to other equivariant representation learning frameworks and discussed applications to the problem of scene mapping. 

Our formalism builds on the assumption that the group of symmetries is known a priori and not inferred from data. This is viable in applications to robotics, but can be problematic in other domains. If a data feature is not taken into account among symmetries, it will formally define distinct orbits. For example, the eventual change in texture for images of rigid objects has to be part of the symmetries in order to still maintain shapes as the only intrinsic classes. A framework where the group structure is learned might be a feasible, although less interpretable alternative to prior symmetry knowledge that constitutes an interesting line of future investigation.
    
\section*{Acknowledgements}
This work was supported by the Swedish Research Council, Knut and Alice
Wallenberg Foundation and the European Research Council (ERC-BIRD-884807).

%% file: sections/appendix.tex
\onecolumn 
\aistatstitle{Equivariant Representation Learning via Class-Pose Decomposition:\\ Supplementary Materials}

\section{Proofs of Theoretical Results} 

\begin{nonumbprop}
The following holds: 

\begin{itemize}
    \item There is an equivariant isomorphism
\begin{equation}
\mathcal{X} \simeq (\mathcal{X} / G ) \times G
\end{equation}
where $G$ acts trivially on the orbits  and via multiplication on itself, i.e., $g \cdot (e, h) = (e, gh)$ for $g,h \in G$, $e \in \mathcal{X} / G$. In other words, each orbit can be identified equivariantly with the group itself. 

\item Any equivariant map $\varphi :  \ (\mathcal{X} / G ) \times G \rightarrow (\mathcal{X} / G ) \times G$ is a right multiplication on each orbit i.e., for each orbit $O \in \mathcal{X}/G$ there is an $h_O \in G$ such that $\varphi(O, g) = (O', g h_O)$ for all $g \in G$. In particular, if $\varphi$ induces a bijection on orbits then it is an isomorphism. 

\end{itemize}
\end{nonumbprop}

\begin{proof}
We start by proving the first statement. Choose a system of representatives $\mathcal{S} \subseteq \mathcal{X}$ for orbits, that is $\mathcal{S}$ contains exactly one element for each class. Consider the map $f \colon \ (\mathcal{X} / G) \times G \rightarrow \mathcal{X}$ given, for $s \in \mathcal{S}$ and $g\in G$, by $f([s], g) = g\cdot s, $
where $[s]$ denotes the orbit of $s$.
It is straightforward to check that $f$ is indeed equivariant. Now if $f([s],g) = f([t], h)$ for $s,t \in \mathcal{S}$ and $g,h \in G$ then $g \cdot s = h \cdot t$ and $s,t$ are thus in the same orbit, which implies $s = t$ because of uniqueness of representatives. But then $h \cdot s = g \cdot s$ and, equivalently, $g^{-1}h \cdot s = s$, from which we deduce $g = h$ since the action is free. This shows that $f$ is injective. Finally, for $x \in \mathcal{X}$, one can write $x = g \cdot s$ for the representative $s$ of the orbit of $x$, which means that $x = f([s], g)$. That is, $f$ is surjective and thus also bijective, which concludes the proof of the first statement.

We now prove the second statement. Consider an equivariant map $\varphi :  \ (\mathcal{X} / G ) \times G \rightarrow (\mathcal{X} / G ) \times G$. For each orbit $O \in (\mathcal{X} / G )$ denote by $h_O \in G$ the element such that $\varphi(O,1) = (O', h_O)$. Then by equivariance $\varphi(O, g) = \varphi(O, g1) =  (O', g h_O)$, as desired. 
\end{proof}

\section{Description of Datasets}
In our experiments we deploy the following datasets, which are also summarized in Table \ref{datasettable}: \\

 \textbf{Sprites}: extracted from \emph{dSprites} (\cite{sprites}). It consists of grayscale images depicting three sprites (heart, square, ellipse) translating and dilating in the pixel plane. The group of symmetries is $G=\mathbb{R}^3$: a factor $\mathbb{R}^2$ translates the sprites in the pixel plane while the last copy of $\mathbb{R}$, which is isomorphic via exponentiation to $\mathbb{R}_{>0}$ equipped with multiplication, acts through dilations. The dataset size is $3 \times 10^4$ and there are three orbits, each corresponding to a sprite. 
 
 \textbf{Shapes}: extracted from \emph{3DShapes} (\cite{shapes}). It consists of colored images depicting four objects (cube, cylinder, sphere, pill) on a background divided into wall, floor and sky. Again, $G=\mathbb{R}^3$ but with the action given by color shifting: each of the factors $\mathbb{R}$ acts by changing the color of the corresponding scene component among object, wall and floor. The dataset size is $4 \times 10^4$ and there are four orbits, each corresponding to a shape. 
 
\textbf{Multi-Sprites}: obtained from Sprites by overlapping images of three colored sprites (with fixed scale). The group of symmetries is $G = \mathbb{R}^6 = \mathbb{R}^2 \times \mathbb{R}^2 \times \mathbb{R}^2$, each of whose factors $\mathbb{R}^2$ translates one of the three sprites in the pixel plane. The added colors endow the sprites with an implicit ordering, which is necessary for the action to be well-defined.  The dataset size is $3 \times 10^4$. Since the scene is composed by three possibly repeating sprites, there are $3^3 = 27$ orbits corresponding to the different configurations.   

\newpage 

\textbf{Chairs}: extracted from \emph{ShapeNet} (\cite{chang2015shapenet}). It consists of colored images depicting three types of chair from different angles. The group of symmetries is $G = \textnormal{SO}(3)$, which rotates the depicted chair. The dataset size is $3 \times 10^4$ and there are three orbits, each corresponding to a type of chair. 

\textbf{Apartments}: extracted from \emph{Gibson} (\cite{xiazamirhe2018gibsonenv}) and generated via the \emph{Habitat} simulator (\cite{habitat19iccv}). It consists of colored images of first-person renderings of two apartments (`Elmira' and `Convoy'). The data simulate the visual perception of an agent such as a mobile robot exploring the two apartments and collecting images and symmetries. The latter belong to the group of two-dimensional orientation-preserving Euclidean isometries $G=\mathbb{R}^2 \times \textnormal{SO}(2)$ and coincide with the possible moves (translations and rotations) by part of the agent. One can realistically imagine the agent perceiving the symmetries through some form of odometry i.e., measurement of movement. Note that the action by $G$ on $\mathcal{X}$ is \emph{partially} defined since $g \cdot x$ is not always possible because of obstacles. As long as the agent is able to reach any part of each of the apartments, the latter still coincide with the two orbits of the group action. The dataset size $2 \times 10^4$. 

All the datasets consist of triples $(x,g,y)$ where $x,y$ are $64 \times 64$ images, $g\in G$ and $y = g \cdot x$.